\documentclass{article}
\usepackage{ijcai23}

\usepackage{times}
\usepackage{soul}
\usepackage{url}
\usepackage[hidelinks]{hyperref}
\usepackage[utf8]{inputenc}
\usepackage[small]{caption}
\usepackage{graphicx}
\usepackage{amsmath}
\usepackage{amsthm}
\usepackage{booktabs}
\usepackage{algorithm}
\usepackage[switch]{lineno}
\usepackage[noend]{algpseudocode}
\usepackage{amssymb}
\usepackage{mathtools}
\usepackage{amsmath}
\usepackage{xcolor}
\usepackage{physics}
\usepackage{todonotes}

\newcommand{\sectref}[1]{Section~\ref{#1}}

\newcommand{\figref}[1]{Figure~\ref{#1}}
\newcommand{\tabref}[1]{Table~\ref{#1}}

\newcommand{\agref}[1]{Algorithm~\ref{#1}}

\newtheorem{proposition}{Proposition}

\newcommand{\startpara}[1]{{\vskip1pt\noindent{\bf #1.}}} 
\renewcommand{\url}[1]{{\def~{\char126}\sf#1}}
\newcommand{\pbox}[1]{\todo[inline, color=blue!10]{#1}}

\def\Rset{\mathbb{R}}

\def\sP{{\mathsf{P}}}

\def\cO{{\mathcal{O}}}

\def\cE{{\mathcal{E}}}

\def\cM{{\mathcal{M}}}
\def\cN{{\mathcal{N}}}
\def\cS{{\mathcal{S}}}
\def\cA{{\mathcal{A}}}
\def\cX{{\mathcal{X}}}

\def\cT{{\mathcal{T}}}

\def\cL{{\mathcal{L}}}
\def\cV{{\mathcal{V}}}

\title{Explainable Multi-Agent Reinforcement Learning for Temporal Queries}

\author{
 Kayla Boggess$^1$\and
 Sarit Kraus$^2$\And
 Lu Feng$^1$\\
 \affiliations
 $^1$University of Virginia\\
 $^2$Bar-Ilan University\\
 \emails
 \{kjb5we, lu.feng\}@virginia.edu, sarit@cs.biu.ac.il
 }

\begin{document}

\maketitle

\begin{abstract}
As multi-agent reinforcement learning (MARL) systems are increasingly deployed throughout society, it is imperative yet challenging for users to understand the emergent behaviors of MARL agents in complex environments. This work presents an approach for generating policy-level contrastive explanations for MARL to answer a temporal user query, which specifies a sequence of tasks completed by agents with possible cooperation. The proposed approach encodes the temporal query as a PCTL$^*$ logic formula and checks if the query is feasible under a given MARL policy via probabilistic model checking. Such explanations can help reconcile discrepancies between the actual and anticipated multi-agent behaviors. The proposed approach also generates correct and complete explanations to pinpoint reasons that make a user query infeasible. We have successfully applied the proposed approach to four benchmark MARL domains (up to 9 agents in one domain). Moreover, the results of a user study show that the generated explanations significantly improve user performance and satisfaction. 

\end{abstract}

\section{Introduction} \label{sec:intro} 

As multi-agent reinforcement learning (MARL) systems are increasingly deployed throughout society, it is imperative yet challenging for users to understand the emergent behaviors of MARL agents in complex environments.
Recent work~\cite{boggess2022toward} proposes methods for generating policy summarization to explain agents' behaviors under a given MARL policy, as well as language explanations to answer user queries about agents' decisions such as ``Why don't [\emph{agents}] do [\emph{actions}] in [\emph{states}]?'' 
However, existing methods cannot handle temporal queries involving a sequence of MARL agents' decisions, for example, ``Why don't [\emph{agents}] complete [\emph{task 1}], followed by [\emph{task 2}], and eventually [\emph{task 3}]?'' 
Explanations to answer such a temporal user query can help reconcile discrepancies between the actual and anticipated agent behaviors.

Recently, there has been increasing interest in generating policy-level contrastive explanations for RL in the single-agent setting.   
\cite{sreedharan2022bridging} considers a problem setting where the agent comes up with a plan to achieve a certain goal, and the user responds by raising a foil (represented as a sequence of agent states and actions). To show why the agent's plan is preferred over the foil (e.g., the foil leads to an invalid state), explanations are generated by finding missing preconditions of the failing foil action on a symbolic model through sample-based trials.
\cite{finkelstein2022deep} considers a similar problem setting, where the user queries about an alternative policy specifying actions that the agent should take in certain states. 
Explanations are defined as a sequence of Markov decision process (MDP) transforms, such that the RL agent's optimal policy (i.e., seeking to maximize its accumulated reward) in the transformed environment aligns with the user queried policy.

There are many challenges and limitations when applying these approaches in multi-agent environments.
First, we need a better representation of user queries. Asking the user to provide concrete information about agents' joint states and joint actions, which grow exponentially with the increasing number of agents, is tedious, if not impractical. 
Further, these approaches have limited scalability in multi-agent environments due to computational complexity. 
\cite{sreedharan2022bridging} requires a large number of samples generated via a random walk to find missing preconditions.  
\cite{finkelstein2022deep} computes a sequence of MDP transforms (e.g., mapping the entire state/action space) and retrains the agent policy in each transformed MDP.
Moreover, the generated explanations may not capture agent cooperation requirements that are essential for understanding multi-agent behaviors. 


We address these challenges by developing an approach to generate policy-level contrastive explanations for MARL. 
Our proposed approach takes the input of a temporal user query specifying which tasks should be completed by which agents in what order.
Any unspecified tasks are allowed to be completed by the agents at any point in time.
The user query is then encoded as a PCTL$^*$ logic formula, which is checked against a multi-agent Markov decision process (MMDP) representing an abstraction of a given MARL policy via \emph{probabilistic model checking}~\cite{KNP17}.
If the MMDP satisfies the PCTL$^*$ formula, then the user query is feasible under the given policy (i.e., there exists at least one policy execution that conforms with the user query).
Otherwise, our approach deploys a guided rollout procedure to sample more of the MARL agents' behaviors and update the MMDP with new samples.  
If the updated MMDP still does not satisfy the PCTL$^*$ formula, the proposed approach generates correct and complete explanations that pinpoint the causes of all failures in the user query.  

Computational experiments on four benchmark MARL domains demonstrate the scalability of our approach (up to 9 agents in one domain). 
It only took seconds to check the feasibility of a user query and generate explanations when needed.
Additionally, we conducted a user study to evaluate the quality of generated explanations, where we adapted~\cite{sreedharan2022bridging} to generate baseline explanations. 
The study results show that, compared with the baseline, explanations generated using our approach significantly improve user performance (measured by the number of correctly answered questions) and yield higher average user ratings on \emph{explanation goodness metrics} (e.g., understanding, satisfaction)~\cite{hoffman2018metrics}.

\section{Related Work} \label{sec:related} 

\subsection{Explainable Reinforcement Learning}
A growing body of research in explainable RL has emerged in recent years, as surveyed in~\cite{wells2021explainable,heuillet2021explainability,puiutta2020explainable}.
Existing works can be categorized according to different axes (e.g., timing, scope, form, setting). We position our proposed approach based on these categorizations as follows.  

First, there are \emph{intrinsic} and \emph{post-hoc} methods depending on the timing when the explanation is generated. The former (e.g., \cite{topin2021iterative,landajuela2021discovering}) builds intrinsically interpretable policies (e.g., represented as decision trees) at the time of training, while the latter (e.g., \cite{sreedharan2022bridging,hayes2017improving}) generates post-hoc explanations after a policy has been trained.
Our proposed approach belongs to the latter.

Second, existing works can be distinguished by the scope of explanations. Some methods provide explanations about policy-level behaviors (e.g., \cite{topin2019generation,amir2018highlights}), while others explain specific, local decisions (e.g., \cite{olson2021counterfactual,madumal2020explainable}). 
Our work focuses on explaining discrepancies between actual and anticipated policy-level behaviors. 

Additionally, current approaches generate explanations in diverse forms, including natural language~\cite{hayes2017improving}, saliency maps~\cite{atrey2019exploratory}, reward decomposition~\cite{juozapaitis2019explainable}, finite-state machines~\cite{danesh2021re}, and others. 
Our proposed approach generates language explanations following~\cite{hayes2017improving} and \cite{boggess2022toward}, both of which use the Quine-McCluskey algorithm to compute a minimized Boolean formula and then translate the formula into an explanation using language templates. 

Finally, the majority of existing works on explainable RL focus on the single-agent setting.
There is very little prior work considering multi-agent environments.
\cite{heuillet2022collective} estimates the contribution of each agent for a group plan, but only as a general explanation of a model and not for a specific instance given by a user.
\cite{boggess2022toward} develops methods to generate policy summarization and query-based language explanations for MARL. However, as discussed in \sectref{sec:intro}, existing methods cannot handle temporal queries considered in this work.

\subsection{Contrastive Explanations}

\cite{miller2019explanation} identifies being \emph{contrastive} (``Why $A$ but not $B$?'') as one of the key desired properties of an explanation.
The research thread on contrastive explanations for RL has been drawing increasing attention since then.
For example, \cite{madumal2020explainable} generates contrastive explanations for \emph{``why action''} and \emph{``why not action''} queries via counterfactual analysis of a structural causal model; 
\cite{lin2021contrastive} develops a deep RL architecture with an embedded self-prediction model to explain why a learned agent prefers one action over another;
and \cite{olson2021counterfactual} computes counterfactual state explanations (i.e., minimal changes needed for an alternative action).
These works all focus on generating contrastive explanations for the RL agent's local decisions in a state.
By contrast, several recent works~\cite{sreedharan2022bridging,finkelstein2022deep} generate policy-level contrastive explanations in the single-agent setting. 
However, as discussed in \sectref{sec:intro}, these methods are not suitable for MARL.
Our proposed approach advances the state of the art by developing an approach for generating contrastive explanations about MARL agents' policy-level behaviors.

\section{Program Formulation} \label{sec:problem} 

We consider a problem setting where a MARL policy has been trained over $N$ agents, denoted by $\pi: \cX \to \Delta (\cA)$, which is a function mapping a set of joint states $\cX=\{\vb{x}=(x^1, \dots, x^N)\}$ to a distribution over a set of joint actions $\cA=\{\vb{a}=(a^1, \dots, a^N)\}$.
Execution of policy $\pi$ yields a sequence of joint states and joint actions 
$\vb{x}_0 \xrightarrow{\vb{a}_0}\vb{x}_1 \xrightarrow{\vb{a}_1} \cdots$
where $\vb{a}_t \sim \pi(\cdot | \vb{x}_t)$ at each step $t$. 
Suppose that the goal of the agents is to jointly complete a set $G$ of tasks (sub-goals). 
Let $R^i: \cX \times \cA \times \cX \to \Rset$ denote the reward function that determines the immediate reward received by agent $i$. 
A positive reward $R^i(\vb{x}_t, \vb{a}_t, \vb{x}_{t+1})>0$ is only received when a task $g \in G$ is completed by agent $i$ at step $t$. 
We assume that each agent can complete at most one task at a step and, if multiple agents cooperate to complete a task, each of them would receive a positive reward at the same step.

To start with, the user is presented with a high-level plan that summarizes one possible execution of the given MARL policy $\pi$. 
For example, consider a MARL domain where three robotic agents are trained to complete search and rescue tasks shown in \figref{fig:domain}(a). We can compute a high-level plan by applying the policy summarization method proposed in~\cite{boggess2022toward}. \figref{fig:domain}(b) illustrates an example plan, where columns indicate the order of tasks completed by agents and each row corresponds to an agent's task sequence.
Agent cooperation is represented by multiple agents sharing the same task in the same column.
In this example, robots II and III first cooperate to fight the fire, followed by robots I and II jointly removing the obstacle, and finally robots I and III rescue the victim together.

\begin{figure}[t]
     \centering
         \includegraphics[width=\columnwidth]{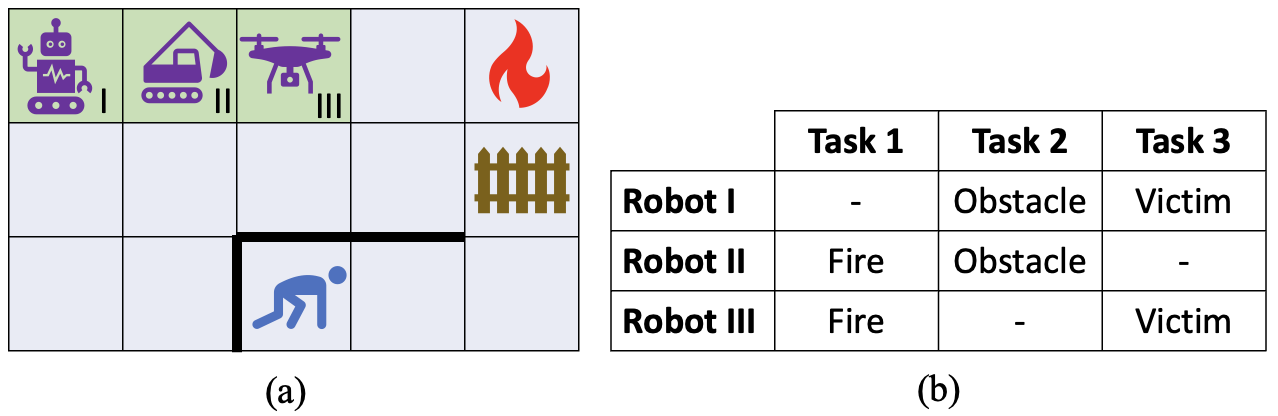}
     \caption{Example MARL domain and a high-level plan.}
     \label{fig:domain}
\end{figure}

The user may not desire the presented plan and raise an alternative query. The user query does not have to be a complete plan involving all agents and tasks. Instead, the user can query about a partial plan such as \emph{``Why don't robots I and II remove the obstacle before robot II fights the fire alone?''}
We define a \emph{temporal user query} as a list of atomic propositions specifying an order of tasks completed by some agents, denoted by $\rho= \langle \tau_1, \tau_2, \cdots \rangle$, where each $\tau$ specifies a task $g \in G$ and designated agents.
Tasks not specified in the query can be completed in any order (e.g., before $\tau_1$, between $\tau_1$ and $\tau_2$, or after $\tau_2$).
The aforementioned example query is denoted by $\langle \mathsf{obstacle\_robotI\_robotII}, \mathsf{fire\_robotII} \rangle$.

A temporal user query $\rho$ is \emph{feasible} under a MARL policy $\pi$ if there exists at least one execution of $\pi$ that conforms with the queried plan $\rho$. When $\rho$ is infeasible under $\pi$, explanations are generated to reconcile discrepancies between the actual and anticipated multi-agent behaviors.
We say that an explanation is \emph{correct} if it pinpoints the causes of one or more failures in $\rho$ (e.g., unsatisfied task preconditions or agent cooperation requirements).
A correct explanation is \emph{complete} if it identifies the reasons behind all failures of a user query $\rho$. 

This work aims to solve the following problem.

\pbox{\textbf{Problem:} Given a temporal user query $\rho$ and a trained MARL policy $\pi$, check if $\rho$ is feasible under policy $\pi$. If $\rho$ is infeasible, generate correct and complete explanations to reconcile discrepancies between the actual and anticipated multi-agent behaviors.}

\section{Approach} \label{sec:approach} 

\begin{algorithm}[b]
\footnotesize
\caption{Checking the feasibility of a user query}
\label{ag:overview}
\textbf{Input}: a temporal user query $\rho$, a trained MARL policy $\pi$ \\
\textbf{Output}: YES, or explanations $\cE$
\begin{algorithmic}[1] 
\State construct a policy abstraction MMDP $\cM$ given $\pi$
\State encode the temporal query $\rho$ as a PCTL$^*$ formula $\varphi$
\If {$\cM$ satisfies $\varphi$}
    \State \Return YES
\Else
    \State $\cM' \gets$ update $\cM$ via guided rollout (\agref{ag:rollout})
    \If {$\cM'$ satisfies $\varphi$}
        \State \Return YES
    \Else
        \State generate explanations $\cE$ (\agref{ag:explanation})
        \State \Return $\cE$
    \EndIf   
\EndIf
\end{algorithmic}
\end{algorithm}

To tackle this problem, we present an approach as illustrated in \agref{ag:overview}.
We describe the construction of a policy abstraction (line 1) in \sectref{sec:mmdp}, the encoding and checking of the user query (lines 2-5) in \sectref{sec:logic}, guided rollout (lines 6-9) in \sectref{sec:rollout}, and explanation generation (lines 10-11) in \sectref{sec:qm}.
Additionally, we analyze the correctness and complexity of the approach in \sectref{sec:properties}.

\subsection{Policy Abstraction MMDP} \label{sec:mmdp}

Given a trained MARL policy $\pi$, we construct a multi-agent Markov decision process (MMDP) abstraction following the policy abstraction method described in~\cite{boggess2022toward}.
We denote an MMDP as a tuple $\cM=(\cS, \vb{s}_0, \cA, \cT, \cL)$ with a set of joint abstract states $\cS$, an initial state $\vb{s}_0 \in \cS$, a set of joint actions $\cA$, a transition function $\cT: \cS \times \cA \to \Delta(\cS)$, and a labeling function $\cL: \cS \to 2^{AP}$ that assigns a set of atomic propositions $AP$ to states.  
A path through $\cM$ is a sequence 
$\vb{s}_0 \xrightarrow{\vb{a}_0}\vb{s}_1 \xrightarrow{\vb{a}_1} \cdots$
starting from the initial state $\vb{s}_0$.

The state space $\cS=\{\vb{s}=(s^1, \dots, s^N)\}$ is defined over a set of Boolean predicates indicating whether a task $g\in G$ has been completed by agent $i$.
The initial state $\vb{s}_0$ represents that none of the tasks has been completed.
In the example MMDP shown in \figref{fig:mmdp}, the initial state is $\vb{s}_0=(000, 000, 000)$. 
State $\vb{s}_1=(000, 100, 100)$ represents that the fire task has been completed by robotic agents II and III, which is labeled with $\cL(\vb{s}_1)=\{\mathsf{fire\_robotII\_robotIII}\}$.
The next state $\vb{s}_2=(010, 110, 100)$ is labeled with $\cL(\vb{s}_2)=\{\mathsf{obstacle\_robotI\_robotII}\}$, which only contains the newly completed obstacle task.

The MMDP transition function $\cT$ is built by finding corresponding abstract transitions $(\vb{s},\vb{a},\vb{s}')$ of each sample $(\vb{x},\vb{a},\vb{x}')$ observed during the MARL policy evaluation, and transition probabilities are computed via frequency counting.
Given a joint state $\vb{x}=(x^1, \dots, x^N)$, we determine a corresponding joint abstract state $\vb{s}=(s^1, \dots, s^N)$ by checking if agent $i$ receives a reward $R^i(\vb{x}, \vb{a}, \vb{x}')>0$ for completing a task $g \in G$. 
For each MMDP state $\vb{s} \in \cS$, we keep track of a set of corresponding sampled joint states, denoted by $X(\vb{s}) = \{\vb{x}\}$, and count the total number of observed MARL samples, denoted by $C(\vb{s})$.

\begin{figure}[t]
     \centering
         \includegraphics[width=0.7\columnwidth]{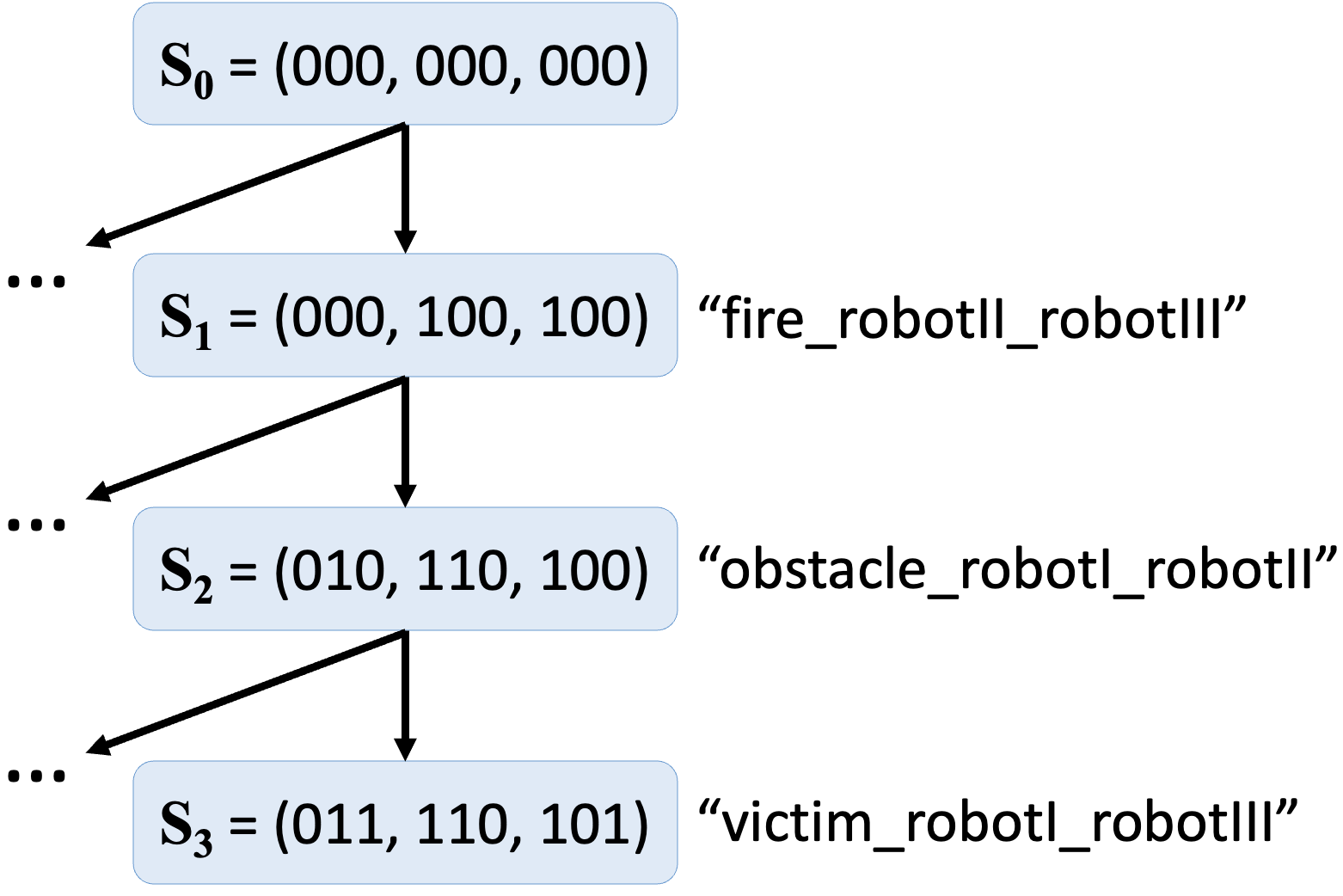}
     \caption{Fragment of an example MMDP.}
     \label{fig:mmdp}
\end{figure}

\subsection{Query Checking with Temporal Logic} \label{sec:logic}

We encode a temporal user query $\rho= \langle \tau_1, \tau_2, \cdots \rangle$ as a PCTL$^*$ logic~\cite{aziz1995usually} formula $\varphi$ with a ``sequencing'' specification template as follows.
$$
\varphi = \sP_{>0} [\lozenge (\tau_1 \wedge \lozenge (\tau_2 \wedge \lozenge \cdots))] 
$$
where $\sP_{>0}$ means that the specification should be satisfied with non-zero probability, 
and $\lozenge$ denotes the logical operator ``eventually''.
The PCTL$^*$ formula $\varphi$ is satisfied in an MMDP $\cM$ if there exists a path through $\cM$ such that $\tau_1$ eventually becomes true at some point along the path, and $\tau_2$ eventually holds at some point afterward.
For example, the MMDP shown in \figref{fig:mmdp} satisfies a PCTL$^*$ formula
$\sP_{>0} [\lozenge (\mathsf{fire\_robotII\_robotIII} \wedge \lozenge \mathsf{victim\_robotI\_robotIII})]$.

To check if $\cM$ satisfies $\varphi$, we apply \emph{probabilistic model checking}~\cite{KNP17} which offers efficient techniques for the exhaustive exploration of $\cM$ to determine if $\varphi$ holds in any path.  
If $\cM$ satisfies $\varphi$, then \agref{ag:overview} returns YES, indicating that the user query is feasible under the given MARL policy. 
Otherwise, there does not exist any path through $\cM$ that conforms with the user query. 
Since the MMDP $\cM$ is constructed based on samples observed during the MARL policy evaluation, it does not necessarily capture all possible agent behaviors under the given policy $\pi$. 
Thus, $\cM$ not satisfying $\varphi$ is not a sufficient condition for claiming that the user query is infeasible under the given MARL policy. 
To address this issue, we develop a guided rollout procedure to update the MMDP $\cM$ via drawing more samples from the MARL policy $\pi$.

\subsection{Guided Rollout} \label{sec:rollout}

\begin{algorithm}[tb]
\footnotesize
\caption{Guided rollout}
\label{ag:rollout}
\textbf{Input}: a trained MARL policy $\pi$, a policy abstraction MMDP $\cM$ \\
\textbf{Output}: an updated MMDP $\cM'$
\begin{algorithmic}[1] 
\State unfold $\cM$ as a search tree and assign a $U$ value to each node
\State $\cN \gets$ tree nodes ordered by $U$ values and sample counts
\For{($k = 0; k < \mathtt{RolloutNum}; k\mbox{++})$}
    \State $\vb{s} \gets \cN.\mbox{pop(0)}$ 
    \State $\vb{x} \gets$ pick a corresponding joint state from $X(\vb{s})$
    \State $\delta \gets$ a rollout execution of $\pi$ from $\vb{x}$ with $\mathtt{DepthLimit}$
    \State update the MMDP with samples in $\delta$
\EndFor
\State \Return the updated MMDP $\cM'$
\end{algorithmic}
\end{algorithm}

\agref{ag:rollout} illustrates the guided rollout procedure, which starts by unfolding paths of the MMDP $\cM$ as a search tree. The root node of the tree is the initial state $\vb{s}_0$ of $\cM$. 
As the search tree unfolds, we assign a $U$ value to each node representing the degree to which the path from the root node to the current node conforms with the user query. 
Consider an example user query 
$\langle \tau_1=\mathsf{fire\_robotII\_robotIII}, \tau_2=\mathsf{obstacle\_robotII} \rangle$,
unfolding the MMDP in \figref{fig:mmdp} yields 
$U(\vb{s}_0)=0$, $U(\vb{s}_1)=1$ for conforming with $\tau_1$, and $U(\vb{s}_2)=-\infty$ for violating $\tau_2$.
The search tree stops expanding a node with $U=-\infty$ since the user query is already violated along the path. 

Let $\cN$ be a queue of tree nodes ordered by decreasing $U$ values and, for nodes with the same $U$ value, increasing counts of MARL samples $C(\vb{s})$.
This ordering prioritizes the exploration of states with a higher degree of user query conformance (i.e., $U$ values) and less sampling.
Given a joint abstract state $\vb{s} \in \cN$, we (randomly) pick a corresponding joint state $\vb{x} \in X(\vb{s})$ and generate a rollout execution $\delta = \vb{x} \xrightarrow{\vb{a}}\vb{x'} \xrightarrow{\vb{a'}} \cdots$ of the policy $\pi$ starting from $\vb{x}$.
The rollout depth $|\delta|$ is bounded by a predefined parameter $\mathtt{DepthLimit}$.
We update the MMDP with samples observed in $\delta$.
Then, we consider the next node in $\cN$ and repeat the above process (lines 4-7 of \agref{ag:rollout}). 
When the number of rollout executions hits a predefined parameter $\mathtt{RolloutNum}$, \agref{ag:rollout} terminates with an updated MMDP, denoted by $\cM'$.

We check if $\cM'$ satisfies a PCTL$^*$ formula $\varphi$ encoding the user query $\rho$ (line 7 of \agref{ag:overview}) as described in \sectref{sec:logic}. 
If $\cM'$ satisfies $\varphi$, then the user query $\rho$ is feasible under the given MARL policy $\pi$. 
When $\cM'$ does not satisfy $\varphi$, the user query is infeasible in the MMDP $\cM'$.
Given sufficiently large $\mathtt{RolloutNum}$ and $\mathtt{DepthLimit}$, the MMDP $\cM'$ provides a good approximation of MARL agents' behaviors under the given policy $\pi$.
Thus, we can claim that the user query $\rho$ is infeasible under $\pi$ with high probability.
In this case, we generate explanations to reconcile discrepancies between the actual and anticipated multi-agent behaviors.

\subsection{Explanation Generation} \label{sec:qm}

\begin{algorithm}[tb]
\footnotesize
\caption{Generating reconciliation explanations}
\label{ag:explanation}
\textbf{Input}: a user query $\rho= \langle \tau_1, \tau_2, \cdots \rangle$, the updated MMDP $\cM'$ \\
\textbf{Output}: explanations $\cE$
\begin{algorithmic}[1] 
\State $\cE \gets \{\}$
\While{$\rho$ is infeasible in $\cM'$}
    \State $U^\mathsf{max} \gets$ the maximum $U$ value in the search tree of $\cM'$
    \State find a failure $\tau_j$ where $j=U^\mathsf{max} + 1$ 
    \State $\cV \gets$ target MMDP states that complete the task in $\tau_j$
    \State $\bar{\cV} \gets$ non-target MMDP states
    \If{$\cV \not= \emptyset$}
    \State $\phi \gets$ Quine-McCluskey(1=$\mathtt{binary}(\cV)$, 0=$\mathtt{binary}(\bar{\cV})$)
    \State $\epsilon \gets$ select a minterm in $\phi$ that is closest to $\rho$
    \EndIf
    \State $\cE \gets$ insert language explanations
    \State update $\rho$ to fix the failure $\tau_j$
\EndWhile
\State \Return $\cE$
\end{algorithmic}
\end{algorithm}

%

\agref{ag:explanation} shows the explanation generation procedure. 
Given the updated MMDP $\cM'$ resulting from \agref{ag:rollout}, we unfold $\cM'$ as a search tree and assign a $U$ value to each tree node following \sectref{sec:rollout}.
Let $U^\mathsf{max}$ denote the maximum $U$ value in the tree.
Then, $\tau_j$ with $j=U^\mathsf{max} + 1$ is a failed task making the query $\rho$ infeasible. 
For example, consider a user query 
$\langle \tau_1=\mathsf{obstacle\_robotI\_robotII}, \tau_2=\mathsf{victim\_robotI}, \tau_3=\mathsf{fire\_robotII\_robotIII} \rangle $, which yields $U^\mathsf{max}=0$ indicating that $\tau_1$ fails.
To pinpoint the cause of this failure, we find a set of target MMDP states $\cV$ where the failed task is completed by some agents (not necessarily by the queried agents).
All other possible states (including those not sampled) are placed in a non-target set $\overline{\cV}$.


When $\cV$ is non-empty, we obtain a minimized Boolean formula $\phi$ by applying the Quine-McCluskey algorithm~\cite{quine1952problem},
which represents the minimal description of the states in the target set $\cV$ compared to those in the non-target set $\overline{\cV}$.
We select a minterm $\epsilon$ in $\phi$ that is closest to $\rho$ (e.g., involving queried agents) and convert $\epsilon$ into an explanation using language templates.
For example, the MMDP state $\vb{s}_2$ in \figref{fig:mmdp} is a target state for $\tau_1$ based on its state label, which indicates that the obstacle task is completed by robots I and II in this state. 
Applying Quine-McCluskey yields a single-minterm formula $\phi = \mathsf{fire\_robotII} \wedge \mathsf{fire\_robotIII} \wedge \mathsf{obstacle\_robotI} \wedge \mathsf{obstacle\_robotII}$.
Recall our assumption in \sectref{sec:problem} that each agent can complete at most one task at a step.
Thus, the fire task must have been completed by robots II and III in some previous state. 
We obtain an explanation: 
\emph{``The robots cannot remove the obstacle because fighting the fire must be completed before removing the obstacle.''}

To generate correct and complete explanations for all possible failures in a user query, we update $\rho$ based on the minterm $\epsilon$ to fix the failure $\tau_j$.
Since $\epsilon$ is the closest minterm to $\rho$, the applied changes are minimal. 
We check whether the updated $\rho$ is feasible in $\cM'$ via probabilistic model checking as described in \sectref{sec:logic}. 
If the model checker yields YES, then \agref{ag:explanation} terminates because all failures of the (original) user query have been explained and fixed. 
Otherwise, the algorithm repeats lines 3-11 for the updated $\rho$.
Following the previous example, 
we update the query as $\langle \tau_1=\mathsf{fire\_robotII\_robotIII}, \tau_2=\mathsf{obstacle\_robotI\_robotII}, \tau_3=\mathsf{victim\_robotI} \rangle $, 
which results in $U^\mathsf{max}=2$, indicating that the updated query still has a failure $\tau_3=\mathsf{victim\_robotI}$.
The MMDP state $\vb{s}_3$ in \figref{fig:mmdp} is a target state where the victim task is completed. 
Applying Quine-McCluskey yields $\phi = \mathsf{victim\_robotI} \wedge \mathsf{victim\_robotIII}$, which only contains one minterm and is translated into an explanation:
\emph{``The robots cannot rescue the victim because Robot I needs Robot III to help rescue the victim.''}
We further update the query as $\langle \tau_1=\mathsf{fire\_robotII\_robotIII}, \tau_2=\mathsf{obstacle\_robotI\_robotII}, \tau_3=\mathsf{victim\_robotI\_robotIII} \rangle $, which is feasible because the MMDP path $\vb{s}_0 \to \vb{s}_1 \to \vb{s}_2 \to \vb{s}_3$ in \figref{fig:mmdp} conforms with this query. 
The algorithm terminates and returns the generated explanations of all failures.

Note that in the special case where the target states set $\cV$ is empty, we skip the Quine-McCluskey and generate an explanation to indicate that the queried task has not been completed in any observed sample. Then, we update the user query by removing the failed task and continue with \agref{ag:explanation}.

\subsection{Correctness and Complexity} \label{sec:properties}

\startpara{Correctness}
The correctness of our proposed approach, with respect to the problem formulated in \sectref{sec:problem}, is stated below and the proof is given in the appendix.  

\begin{proposition}
Given a temporal user query $\rho$ and a trained MARL policy $\pi$, if \agref{ag:overview} returns YES, then the query $\rho$ must be feasible under the policy $\pi$; otherwise, \agref{ag:overview} generates correct and complete explanations $\cE$. 
\label{thm:correctness}
\end{proposition}

\startpara{Complexity}
We analyze the complexity of the following key steps in the proposed approach.
\begin{itemize}
    \item The time complexity of checking an MMDP against a PCTL$^*$ formula $\varphi$ defined in \sectref{sec:logic} via probabilistic model checking is double exponential in $|\varphi|$ (i.e., equal to the length of the user query $|\rho|$) and polynomial in the size of the MMDP~\cite{baier2008principles}. The MMDP state space size $|\cS|$ is bounded by $\cO({2^{|G|}}^N)$, depending on the number of agents $N$ and tasks $|G|$. However,  
    only a small set of reachable states is usually induced in practice (as shown in \tabref{tab:exp}), given a well-trained MARL policy. 
    \item The time complexity of guided rollout (\agref{ag:rollout}) is given by $\cO\big(\mathtt{RolloutNum} \cdot \mathtt{DepthLimit})$. As discussed above, the larger the parameter values of $\mathtt{RolloutNum}$ and $\mathtt{DepthLimit}$, the better approximation of MARL policy behaviors captured by the updated MMDP $\cM'$.
    \item The time complexity of explanation generation (\agref{ag:explanation}) is given by $\cO\big(\lambda \cdot (3^{N \cdot |G|} / \sqrt{ N \cdot |G|})\big)$, where $\lambda$ is the number of failures in the user query, and $\cO\big(3^{N \cdot |G|} / \sqrt{ N \cdot |G|}\big)$ is the time complexity of Quine-McClusky \cite{chandra1978number}. 
\end{itemize}

Even though the complexity is high, in practice it is possible to check query feasibility and generate explanations in reasonable times as shown in the next section.

\section{Computational Experiments} \label{sec:exp} 

To demonstrate the scalability of our approach, we developed a prototype implementation and applied it to four benchmark MARL domains
~\footnote{Code available at \href{https://github.com/kjboggess/ijcai23}{github.com/kjboggess/ijcai23}}.
\begin{enumerate}
    \item[(1)] \emph{Search and Rescue (SR)}, where multiple robotic agents cooperate to complete tasks such as fighting fires and rescuing victims~\cite{boggess2022toward}.
    \item[(2)] \emph{Level-Based Foraging (LBF)}, where agents play a mixed cooperative-competitive game to collect food scattered in a gridworld~\cite{papoudakis2021benchmarking}.
    \item[(3)] \emph{Multi-Robot Warehouse (RWARE)}, where robots collaboratively move and deliver requested goods~\cite{papoudakis2021benchmarking}. 
    \item[(4)] \emph{PressurePlate (PLATE)}, where agents are required to cooperate during the traversal of a gridworld, with some agents staying on pressure plates to open the doorway for others to proceed~\cite{pressureplate}.

\end{enumerate}

Our prototype implementation used the Shared Experience Actor-Critic~\cite{christianos2020shared} for MARL policy training and evaluation.
All models were trained and evaluated until converging to the expected reward, or up to 10,000 steps, whichever occurred first. 
The PRISM probabilistic model checker~\cite{KNP11} was applied for checking the feasibility of user queries. 
We set the guided rollout parameters as $\mathtt{RolloutNum}=10$ and $\mathtt{DepthLimit}=50$.
The experiments were run on a machine with 2.1 GHz Intel CPU, 132 GB of memory, and CentOS 7 operating system.

\tabref{tab:exp} shows experimental results. For each case study, we report the number of agents $N$, the number of tasks $|G|$, and the length of user queries $|\rho|$.
Additionally, we report the size of policy abstraction MMDPs $\cM$ in terms of the number of (reachable) states $|\cS|$ and the number of transitions $|\cT|$.
In general, the MMDP size increases with a growing number of agents and tasks.
However, an unequal distribution of agent actions under the MARL policy $\pi$ can lead to a smaller MMDP $\cM$ (e.g., LBF-5) as agents take the same trajectories more often leading to less exploration.

\begin{table}[t]
\resizebox{\columnwidth}{!}{%
\begin{tabular}{cccclrrlrlcr}
\toprule
\multicolumn{4}{c}{\textbf{Case Study}} &  & \multicolumn{2}{c}{\textbf{MMDP $\cM$}} &  & \multicolumn{1}{c}{\textbf{Feasible}} &  & \multicolumn{2}{c}{\textbf{Infeasible}} \\ \cmidrule{1-4} \cmidrule{6-7} \cmidrule{9-9} \cmidrule{11-12} 
Domain & $N$ & $|G|$ & $|\rho|$ &  & \multicolumn{1}{c}{$|\mathcal{S}|$} & \multicolumn{1}{c}{$|\mathcal{T}|$} & \multicolumn{1}{c}{} & \multicolumn{1}{c}{Time (s)} & \multicolumn{1}{c}{} & $\lambda$ & \multicolumn{1}{c}{Time (s)} \\ \cmidrule{1-4} \cmidrule{6-7} \cmidrule{9-9} \cmidrule{11-12} 
 & 3 & 3 & 3 &  & 28 & 127 &  & 0.8 &  & 1 & 2.2 \\
SR & 4 & 4 & 4 &  & 163 & 674 &  & 1.5 &  & 2 & 5.3 \\
 & 5 & 5 & 5 &  & 445 & 1,504 &  & 24.4 &  & 3 & 89.8 \\ \midrule
 & 3 & 3 & 3 &  & 67 & 344 &  & 0.9 &  & 1 & 2.9 \\
LBF & 4 & 4 & 4 &  & 211 & 781 &  & 2.1 &  & 2 & 7.6 \\
 & 5 & 5 & 5 &  & 152 & 454 &  & 4.5 &  & 3 & 20.5 \\ \midrule
 & 2 & 4 & 3 &  & 98 & 268 &  & 0.8 &  & 1 & 15.5 \\
RWARE & 3 & 6 & 5 &  & 442 & 1,260 &  & 3.7 &  & 2 & 42.2 \\
 & 4 & 8 & 8 &  & 1,089 & 2,751 &  & 21.7 &  & 3 & 85.2 \\ \midrule
 & 5 & 3 & 3 &  & 87 & 181 &  & 0.8 &  & 1 & 3.0 \\
PLATE & 7 & 4 & 4 &  & 85 & 175 &  & 0.9 &  & 2 & 25.7 \\
 & 9 & 5 & 5 &  & 132 & 266 &  & 1.4 &  & 3 & 126.8 \\
  \bottomrule
\end{tabular}%
}
\caption{Experimental results on four benchmark MARL domains.}
\label{tab:exp}
\end{table}

We consider two temporal queries (i.e., a feasible query and an infeasible query with the same length $|\rho|$) in each case study and report the runtime of \agref{ag:overview}.
For infeasible queries, we also report the number of failures $\lambda$, which were controlled to grow with the environment size as the longer the query length $|\rho|$, the larger number of task failures it may contain.
The size of generated explanations is equal to the number of failures (i.e., one for each task failure in the user query).
Experimental results show that all queries were solved efficiently within seconds. 
Checking an infeasible query is generally slower than checking a feasible query in the same case study, due to the extra time needed for guided rollout and generating explanations. 

In summary, computational experiments demonstrate that our approach can be successfully applied to various benchmark MARL domains with a large number of agents (e.g., up to 9 agents in the PLATE domain), for checking the feasibility of temporal user queries and generating explanations when needed.

\section{User Study} \label{sec:study} 

We evaluate the quality of generated reconciliation explanations via a user study.~\footnote{This study was approved by University of Virginia Institutional Review Boards IRB-SBS \#5226. } 
\sectref{sec:design} describes the study design and \sectref{sec:results} analyzes the results.

\subsection{Study Design}\label{sec:design}

\startpara{User interface}
The study was conducted via the Qualtrics survey platform. 
Instead of allowing participants to raise queries in real-time, we generated explanations for a selected set of temporal queries \emph{a priori}, which enables us to present the same set of explanations to different participants.  
\figref{fig:interface} shows an example of the user interface.
Participants were shown the agents' original plan (Plan A) and an alternate plan representing a temporal query (Plan B).
An explanation was presented to explain why Plan B was infeasible.
Participants were then asked to use the provided explanation to decide if a new query (Plan C) was feasible.
Participants were incentivized with bonus payments to answer the question correctly.

\begin{figure}[t]
     \centering
         \includegraphics[width=\columnwidth]{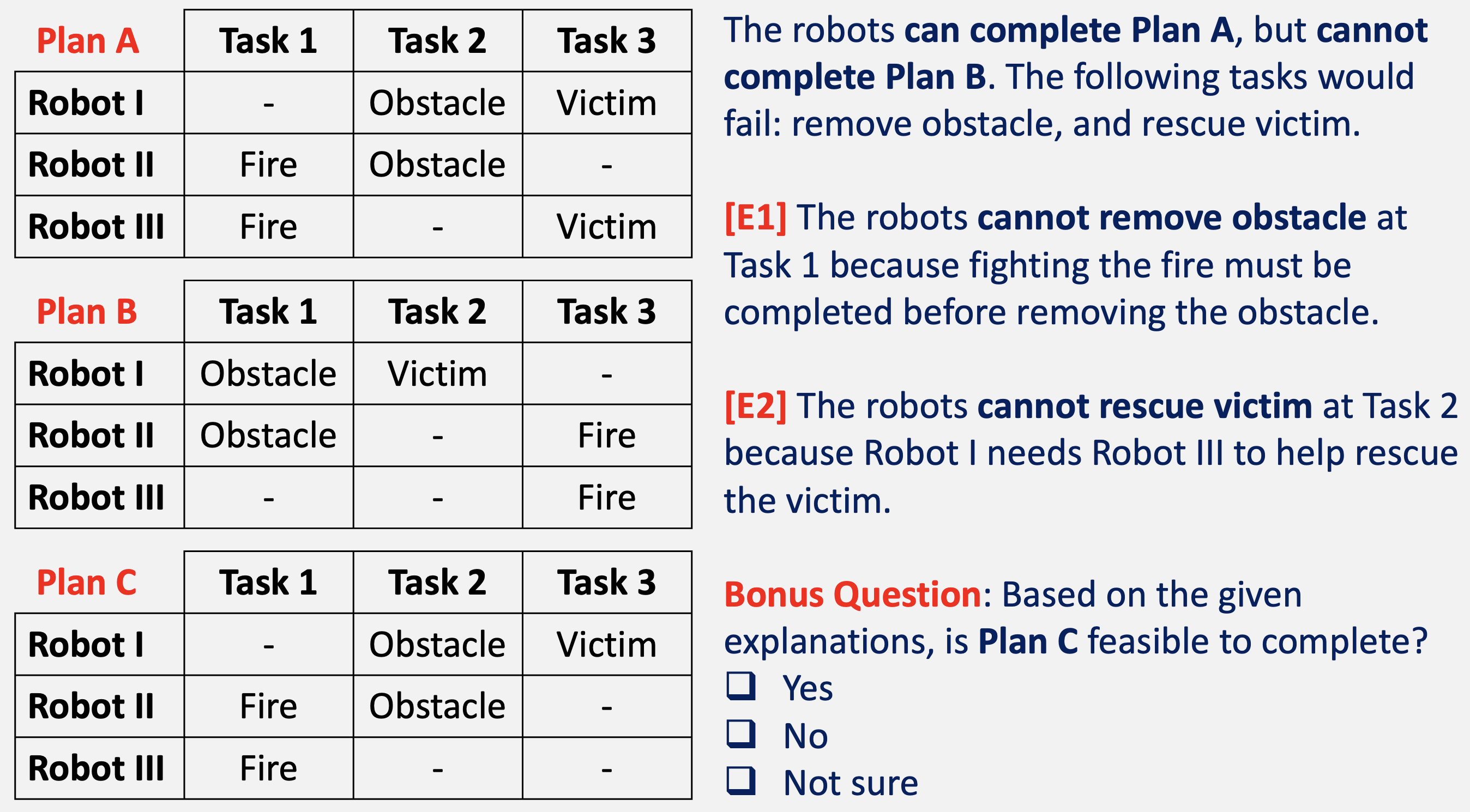}
     \caption{Example of the user study interface displaying explanations generated by the proposed approach.}
     \label{fig:interface}
\end{figure}

\startpara{Participants}
We recruited 88 participants (i.e., fluent English speakers over the age of 18) through university mailing lists (52\% male, 45.5\% female, 2.3\% non-binary).  
They had an average age of 23.9 (SD = 6.1).
To ensure data quality, a demonstration was given, attention checks were injected, and the time to complete the survey was tracked.

\startpara{Baseline}
We adapted the explanation generation method in~\cite{sreedharan2022bridging}, which was initially proposed for the single-agent setting, as a baseline for comparison. 
We extended the method for joint states and actions and limited its sampling to the given policy instead of the larger environment.
Furthermore, we use the same user interface as shown in \figref{fig:interface} to avoid any confounding variables regarding presentation in the study.
The baseline method takes the input of a user query expressed as a sequence of agent states and actions, for which we converted a high-level plan (e.g., Plan B in \figref{fig:interface}) into a low-level execution of joint states and joint actions. 
Explanations generated using the baseline method could fail to capture agent cooperation requirements in multi-agent environments. 
Moreover, the baseline method only provides explanations for the first point of failure rather than all failures in a user query. 
For example, the baseline explanations for Plan B in \figref{fig:interface} changes the second sentence in the explanation to ``The first failed task would be: remove obstacle.'' and only contains E1.
Participants would not be able to answer the bonus question correctly without knowing E2. 

\startpara{Independent variables}
We employed a within-subject study design where participants were asked to complete two trials for evaluating explanations generated using the baseline method and our proposed approach, respectively. 
There were 4 sets of temporal queries (i.e., two single-failure queries and two with multiple failures) and bonus questions in each trial. 
The queried plans and questions used in the two trials were different but had a similar difficulty level. 
Participants were presented with the same set of plans and questions and were randomly assigned to two groups (i.e., evaluating the baseline explanations before or after the proposed explanations) to counterbalance the ordering confound effect.

\startpara{Dependent measures}
We counted the number of questions correctly answered by participants as a performance measure. 
Additionally, at the end of each trial, participants were instructed to rate on a 5-point Likert scale (1 - strongly disagree, 5 - strongly agree) the following statements regarding \emph{explanations good metrics} adapted from~\cite{hoffman2018metrics}.
\begin{itemize}
    \item The explanations help me \emph{understand} how the robots complete the mission.
    \item The explanations are \emph{satisfying}.
    \item The explanations are sufficiently \emph{detailed}.
    \item The explanations are sufficiently \emph{complete}, that is, they provide me with all the needed information to answer the questions.
    \item The explanations are \emph{actionable}, that is, they help me know how to answer the questions.
    \item The explanations let me know how \emph{reliable} the robots are for completing the mission.
    \item The explanations let me know how \emph{trustworthy} the robots are for completing the mission.
\end{itemize}

\startpara{Hypotheses}
We tested two hypotheses stated below. 
\begin{itemize}
    \item \textbf{H1:} Explanations generated by our proposed approach \emph{enable participants to answer more questions correctly} than the baseline explanations.  
    \item \textbf{H2:} Explanations generated by our proposed approach \emph{lead to higher ratings on explanation goodness metrics} than the baseline explanations. 
\end{itemize}

\subsection{Results Analysis}\label{sec:results}

\startpara{Question-answering performance}
Participants were able to answer more questions correctly based on explanations generated by our proposed approach (M=3.1 out of 4, SD=1.0) than those generated with the baseline method (M=0.6 out of 4, SD=0.8).
A paired t-test ($\alpha=0.05$) shows a statistically significant difference ($t$(87)=-17.0, $p\le$0.01, $d$=1.8). 
\emph{Thus, the data supports H1.}

Recall that the baseline method only provides explanations for the first point of failure in a user query and could not always correctly identify agent cooperation requirements.
By contrast, our approach generates correct and complete explanations for all failures in a user query, which help participants to better understand agent behaviors under a given policy, and thus, leads to better question-answering performance. 

\startpara{Explanation goodness ratings}
\figref{fig:Goodfig} shows that participants gave higher subjective ratings to the proposed explanations than the baseline explanations on average, with respect to \emph{all} explanation goodness metrics.

\begin{figure}[tb]
    \centering
    \includegraphics[width=0.8\columnwidth]{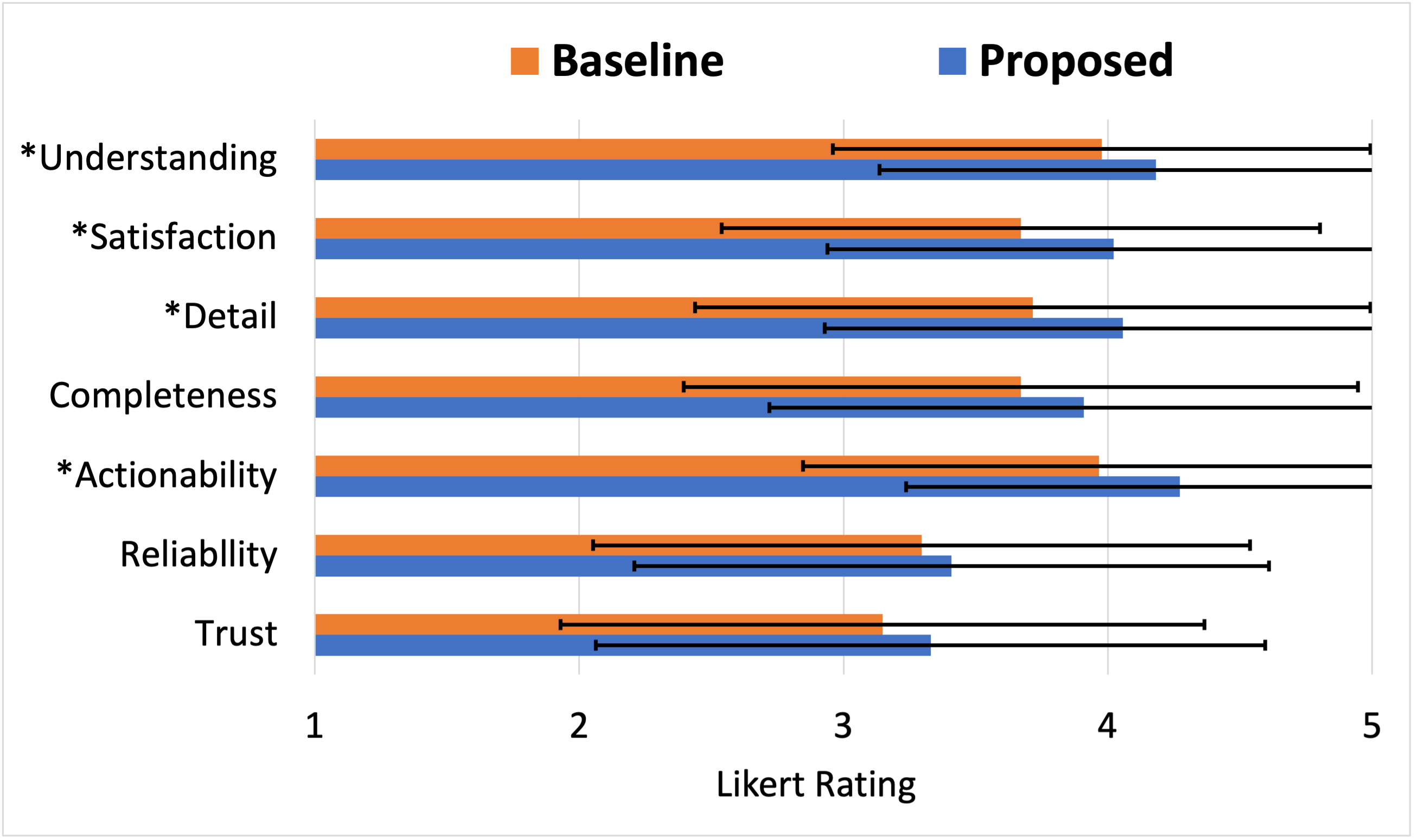}
    \caption{Mean and SD of participant ratings on explanation goodness metrics (``*'' indicates statistically significant difference with the significant level set as $\alpha=0.05$).}
    \label{fig:Goodfig}
    \vspace{-10pt}
\end{figure}

We used the Wilcoxon signed-rank test ($\alpha=0.05$) to evaluate hypothesis H2.
Statistically significant differences were found for the following four metrics:
\textit{understanding} ($W$=315.0, $Z$=-1.6, $p \leq$0.05, $r$=-0.1), 
\textit{satisfaction} ($W$=236.0, $Z$=-2.2, $p \leq$0.01, $r$=-0.2), 
\textit{detail} ($W$=255.0, $Z$=-1.6, $p \leq$0.01, $r$=-0.1), 
and \textit{actionability} ($W$=105.5, $Z$=-2.0, $p \leq$0.02, $r$=-0.1). 
But no significant difference was found on other metrics:
\textit{completeness} ($W$=389.5, $Z$=-1.2, $p \leq$ 0.1, $r$=-0.1),
\textit{reliability} ($W$=255.5, $Z$=-0.5, $p \leq$0.4, $r$=-0.04),
and \textit{trust} ($W$=181.5, $Z$=-1.0, $p \leq$0.07, $r$=-0.1).
\emph{Thus, the data partially supports H2.}

Participants' superior question-answering performance is consistent with their statistically significant higher subjective ratings on understanding, detail, and actionability (i.e., the proposed explanations provide detailed and actionable information for answering questions).
Furthermore, the baseline explanations were rated significantly less satisfying, because they may miss essential information (e.g., agent cooperation) for answering questions.
Participants may misjudge the explanations’ completeness as they were unaware of the total number of failures in a queried plan. 
Finally, the generated explanations are mostly about missing task preconditions, which are less useful for participants to judge how reliable and trustworthy the robots are for completing the mission. 

\startpara{Summary}
Results of the user study show that, compared with the baseline, explanations generated by our proposed approach significantly improve participants' performance in correctly answering questions, and lead to higher average ratings on explanation goodness metrics such as understanding and satisfaction.

\section{Conclusion} \label{sec:conclu} 

This work presents an approach for generating policy-level contrastive explanations for MARL to answer a temporal user query, which specifies a sequence of tasks to be completed by agents with possible cooperation. 
The proposed approach checks if the user query is feasible under the given MARL policy and, if not, generates correct and complete explanations to pinpoint reasons that make a user query infeasible.
A prototype implementation of the proposed approach was successfully applied to four benchmark MARL domains with a large number of agents (e.g., up to 9 agents in one domain).
In all the experiments, it only took seconds to check the feasibility of a user query and generate explanations when needed.
Additionally, a user study was conducted to evaluate the quality of generated explanations. The study results show that explanations generated using the proposed approach can help improve user performance, understanding, and satisfaction. 

There are several directions to explore for possible future work. 
First, we will evaluate the proposed approach with different MARL methods. 
While the prototype implementation only uses one MARL algorithm, the proposed approach should be compatible with any MARL method because it only relies on sampling possible MARL executions.
Second, we will leverage the expressiveness of PCTL$^*$ logic and investigate a richer set of user queries. For example, a ``coverage'' query which specifies a set of tasks to be covered in any order, and a ``sequencing with avoidance'' query which asks for the completion of a sequence of tasks while avoiding some other tasks to be completed by specific agents.
Lastly, we would like to apply the proposed approach to a wide range of MARL environments in real-world scenarios.

\clearpage

\section*{Acknowledgments}
This work was supported in part by U.S. National Science Foundation under grant CCF-1942836, 
U.S. Office of Naval Research under grant N00014-18-1-2829, 
U.S. Air Force Office of Scientific Research under grant FA9550-21-1-0164,
Israel Science Foundation under grant 1958/20, 
and the EU Project TAILOR under grant 952215.
Any opinions, findings, and conclusions or recommendations expressed in this material are those of the author(s) and do not necessarily reflect the views of the grant sponsors.

\appendix
\setcounter{proposition}{0}

\section{Appendix}

\begin{proposition}
Given a temporal user query $\rho$ and a trained MARL policy $\pi$, if \agref{ag:overview} returns YES, then the query $\rho$ must be feasible under the policy $\pi$; otherwise, \agref{ag:overview} generates correct and complete explanations $\cE$.  
\label{thm:correctness}
\end{proposition}
\begin{proof}
We prove the following two cases. 

\noindent \emph{Case 1:} When \agref{ag:overview} returns YES, the policy abstraction MMDP $\cM$ or the updated MMDP $\cM'$ satisfies the PCTL$^*$ formula $\varphi$ encoding the user query $\rho$, indicating that there must exist a path through $\cM$ or $\cM'$ that conforms with $\rho$. 
By construction, every abstract MMDP transition $(\vb{s},\vb{a},\vb{s}')$ in $\cM$ or $\cM'$ with non-zero probability maps to at least one sampled decision $(\vb{x},\vb{a},\vb{x}')$ of the given MARL policy $\pi$. 
Thus, there must exist an execution of policy $\pi$ that conforms with the user query $\rho$. 
By definition, the user query $\rho$ is feasible under the given MARL policy $\pi$.

\noindent \emph{Case 2:} \agref{ag:overview} returns explanations $\cE$ generated via \agref{ag:explanation}.
As described in \sectref{sec:qm}, \agref{ag:explanation} terminates when all failures in the user query $\rho$ have been explained and fixed. 
Given a finite-length temporal query $\rho$, there is a finite number of failures. 
For any failure in the query, if the target states set $\cV$ is non-empty, then the failure must be fixable using a Quine-McCluskey minterm that represents a target state where the failed task is completed. 
If $\cV$ is empty, then the failure is removed from the query. 
Thus, the termination of \agref{ag:explanation} is guaranteed. 
By definition, the generated explanations are \emph{correct} (i.e., identifying the causes of one or more failures in $\rho$) and \emph{complete} (i.e., finding the reasons behind all failures in $\rho$).
\end{proof}


\bibliographystyle{named}
\bibliography{references}

\end{document}